\documentclass[sigconf,natbib,screen]{acmart}

\usepackage[utf8]{inputenc}
\usepackage{amsmath}
\usepackage{acronym}
\usepackage{booktabs}
\usepackage{xcolor}
\usepackage{algorithm}
\usepackage{algpseudocode}
\usepackage{bm}
\usepackage{mathtools}
\usepackage{tikz}
\usetikzlibrary{arrows}
\usepackage[inline]{enumitem}
\usepackage[skip=0pt]{caption}

\usepackage[utf8]{inputenc}

\author{Rolf Jagerman}
\affiliation{%
  \institution{University of Amsterdam}
  \city{Amsterdam}
  \country{The Netherlands}
}
\email{rolf.jagerman@uva.nl}

\author{Maarten de Rijke}
\affiliation{
  \institution{University of Amsterdam \& Ahold Delhaize}
  \city{Amsterdam}
  \country{The Netherlands}
}
\email{m.derijke@uva.nl}

\renewcommand{\shortauthors}{Rolf Jagerman and Maarten de Rijke}

\copyrightyear{2020}
\acmYear{2020}
\setcopyright{acmlicensed}
\acmConference[SIGIR '20]{Proceedings of the 43rd International ACM SIGIR Conference on Research and Development in Information Retrieval}{July 25--30, 2020}{Virtual Event, China}
\acmBooktitle{Proceedings of the 43rd International ACM SIGIR Conference on Research and Development in Information Retrieval (SIGIR '20), July 25--30, 2020, Virtual Event, China}
\acmPrice{15.00}
\acmDOI{10.1145/3397271.3401069}
\acmISBN{978-1-4503-8016-4/20/07}



\acrodef{LTR}{Learning to Rank}
\acrodef{IR}{Information Retrieval}
\acrodef{CLTR}{Counterfactual Learning to Rank}
\acrodef{CRM}{Counterfactual Risk Minimization}
\acrodef{IPS}{Inverse Propensity Scoring}
\acrodef{SGD}{Stochastic Gradient Descent}

\newcommand{\OurMethod}{\textsc{CounterSample}}
\newcommand{\Yahoo}{\texttt{Yahoo}}
\newcommand{\Istella}{\texttt{Istella-s}}

\newcommand{\IPSSGD}{IPS-SGD}
\newcommand{\BiasedSGD}{Biased-SGD}

\DeclarePairedDelimiter{\norm}{\lVert}{\rVert}
\DeclarePairedDelimiter{\iprod}{\langle}{\rangle}
\DeclarePairedDelimiter{\sbrackets}{[}{]}
\DeclarePairedDelimiter{\cbrackets}{\{}{\}}
\DeclarePairedDelimiter{\rbrackets}{(}{)}
\newcommand{\expectation}[1]{\mathbb{E}\sbrackets*{#1}}
\newcommand{\BigO}[1]{O\rbrackets*{#1}}

\newtheorem{theorem}{Theorem}

\allowdisplaybreaks

\title{Accelerated Convergence for Counterfactual Learning to Rank}

\begin{document}
\fancyhead{}


\begin{abstract}
Counterfactual \ac{LTR} algorithms learn a ranking model from logged user interactions, often collected using a production system.
Employing such an offline learning approach has many benefits compared to an online one, but it is challenging as user feedback often contains high levels of bias.
Unbiased \ac{LTR} uses \acf{IPS} to enable unbiased learning from logged user interactions.
One of the major difficulties in applying \ac{SGD} approaches to counterfactual learning problems is the large variance introduced by the propensity weights.
In this paper we show that the convergence rate of \ac{SGD} approaches with \ac{IPS}-weighted gradients suffers from the large variance introduced by the \ac{IPS} weights: convergence is slow, especially when there are large \ac{IPS} weights.

To overcome this limitation, we propose a novel learning algorithm, called \OurMethod, that  has provably better convergence than standard \ac{IPS}-weighted gradient descent methods.
We prove that \OurMethod{} converges faster and complement our theoretical findings with empirical results by performing extensive experimentation in a number of biased \ac{LTR} scenarios -- across optimizers, batch sizes, and different degrees of position bias.
\end{abstract}

\let\backupauthors\authors
\renewcommand{\authors}{Rolf Jagerman and Maarten de Rijke}

\maketitle


\section{Introduction}

\acf{LTR} from \emph{user interactions}~\citep{grotov2016online}, as opposed to learning from \emph{annotated datasets}~\citep{liu2009learning}, has seen increased research interest due to its immense practical value.
Learning from implicit feedback enjoys several advantages over learning from professional annotations:
\begin{enumerate*}
\item user interaction signals are available at large scale and cost much less than professional annotations,
\item implicit feedback captures the user's true interest more accurately, and
\item interaction data can be utilized in domains where professional annotations are impractical, unethical, or impossible, for example in personal search.
\end{enumerate*}
However, learning from user interactions is not without difficulty and one of the major challenges is the biased nature of user interactions~\cite{joachims2017accurately,craswell2008experimental}.
For example, in \ac{LTR}, one of the most important types of bias that affects user interaction data is \emph{position bias}, a phenomenon where users observe, and as a result click on, top-ranked items more than lower ranked ones.

Recent work has focused on removing bias by applying methods from counterfactual learning~\citep{joachims2017unbiased}.
Most notably, \acf{IPS} is commonly used to perform unbiased learning.
A widely used approach for unbiased learning is to treat inverse propensity scores as weights and solve a weighted optimization problem via \acf{SGD}~\cite{ai2018unbiased,joachims2018deep,agarwal2019estimating}.

A major challenge of learning with \ac{IPS}-weighted \ac{SGD} is that the propensity weights introduce a large amount of variance in the gradients.
This effect is especially severe in scenarios where the propensity weights can take on extreme values.
For example, in product search, the set of candidate results tends to be large~\cite{anwaar2019counterfactual} and the query distribution can be heavily skewed~\cite{hasan2011query} (e.g., due to periodic or seasonal influences), which means that some items get very little exposure and as a result have extreme \ac{IPS} weights.

In this paper we investigate the relationship between \ac{IPS} weights and the convergence rate of \ac{IPS}-weighted \ac{SGD}.
We prove that \ac{IPS}-weighted \ac{SGD} suffers from a slow convergence rate when the propensity weights are large.
We also argue that as long as stochastic gradients are scaled with \ac{IPS}-weights, this slowdown \emph{cannot} be improved.
This means that for many practical \ac{LTR} scenarios, learning a ranking model is inefficient and convergence is slow.

We overcome the above limitation with a novel sample-based learning algorithm, called \OurMethod{}, that samples learning instances proportional to their \ac{IPS} weight instead of weighting learning instances by their \ac{IPS} weight.
Because of this strategy, we are able to control the variance, which, in turn, leads to accelerated convergence.
We show that this new approach provably enjoys faster convergence while remaining unbiased and computationally cheap.
We complement these theoretical findings with extensive experiments.
\OurMethod{} consistently converges faster, and in some cases learns a better ranker than \ac{IPS}-weighted \ac{SGD}, in a number of biased \ac{LTR} scenarios -- across optimizers, across batch sizes and for different severities of position bias.

Our main contributions in this paper are:
\begin{itemize}[leftmargin=*,nosep]
\item We analyze the convergence rate of \ac{IPS}-weighted \ac{SGD} algorithms and formalize the relationship between \ac{IPS} weights and the convergence rate.
\item We introduce a novel learning algorithm for Counterfactual \acl{LTR} called \OurMethod{} that has provably faster convergence than \ac{IPS}-weighted \ac{SGD} approaches.
\item We empirically show that \OurMethod{} converges faster than competing methods in a number of biased \ac{LTR} scenarios.
\end{itemize}

\noindent%
The remainder of this paper is organized as follows.
We discuss related work in Section~\ref{sec:related-work}.
Section~\ref{sec:background} provides the necessary notation and background for Counterfactual \ac{LTR}.
We analyze the convergence rate of \ac{IPS}-weighted \ac{SGD} algorithms for \ac{LTR} in Section~\ref{sec:analysis}.
Then, we introduce \OurMethod{}, a method for Counterfactual \ac{LTR} with faster convergence in Section~\ref{sec:method}.
Sections~\ref{sec:setup} and~\ref{sec:results} describe the experimental setup and empirical results, respectively.
We conclude the paper in Section~\ref{sec:conclusion}.


\section{Related work}
\label{sec:related-work}

\subsection{Counterfactual \ac{LTR}}
Recent work on unbiased \acf{LTR} uses counterfactual learning~\cite{swaminathan2015counterfactual,swaminathan2015batch} to remove different types of bias such as position bias from click data to improve ranking performance~\cite{joachims2017unbiased,wang2016learning}.
These methods typically use \acf{IPS} to enable unbiased learning.
A popular approach for solving \ac{IPS}-weighted learning problems is to use \acf{SGD} algorithms~\cite{ai2018unbiased,joachims2018deep,agarwal2019estimating}.
Existing approaches accomplish this by scaling the loss (and as a result, the gradients) with \ac{IPS} weights.
Although the empirical success of such approaches has been well documented in the literature~\cite{agarwal2019estimating,ai2018unbiased,joachims2018deep}, the impact of \ac{IPS} weights on the \emph{convergence rate} of \ac{SGD} is not well understood.

In this paper we address this problem by investigating the relationship between \ac{IPS} weights and the convergence rate of \ac{SGD}.
Furthermore, we introduce a novel learning method that, unlike previous approaches, does \emph{not} scale the loss or gradients, but instead guarantees unbiasedness by employing a sampling procedure.

\subsection{Position Bias}
An important aspect of unbiased \ac{LTR} is estimating the observation probabilities (often called \emph{propensities}), which are necessary to apply \ac{IPS}-weighting.
For example, \citet{wang2018position} propose result randomization strategies for obtaining observation probabilities under a position bias user model.
Recent work has focused on \emph{intervention harvesting}, a less invasive method for propensity estimation~\cite{agarwal2019estimating,fang2019intervention}.

In our work, we do \emph{not} focus on the propensity estimation aspect of unbiased \ac{LTR}, instead assuming the propensity scores are known a priori, because our goal is to study the \emph{convergence rate} of \ac{IPS}-weighted \ac{SGD} algorithms.

\subsection{Convergence Rates for \ac{SGD}}
There is a significant amount of work studying upper bounds for the convergence rate of \ac{SGD}~\cite{robbins1951stochastic} under varying assumptions of convexity and smoothness of the optimization objective~\cite{shamir2013stochastic,hazan2014beyond,rakhlin2011making,hazan2007logarithmic}.
The convergence rate proofs presented in this paper build on proofs provided by~\citet{shalev2014understanding}.
However, unlike previous work, our work investigates the role of \ac{IPS} weights in optimization problems.
We note that there is work investigating the use of importance sampling for \ac{SGD} algorithms~\cite{zhao2015stochastic, alain2015variance, needell2014stochastic, katharopoulos2018not}.
These approaches all start from an unbiased dataset and then improve the convergence rate of \ac{SGD} by manually perturbing the sampling distribution during learning, which introduces bias, and then applying \ac{IPS}-weighting to remove the introduced bias.

Our work is different from these approaches because we learn from an \emph{already biased} click log dataset, where the source of bias is outside of our control (e.g., position bias coming from users).
This means we do not assume control over how the dataset is generated nor do we assume control over the propensity scores.


\section{Background}
\label{sec:background}

We consider the problem of Counterfactual \acf{LTR} as described in~\cite{joachims2017unbiased}.
First, we introduce our notation for \ac{LTR} with additive metrics.
After that, we describe Counterfactual \ac{LTR} from biased click feedback and present the \ac{IPS}-weighted \ac{SGD} algorithm that is commonly used for Counterfactual \ac{LTR}.

\subsection{Learning to Rank with Additive Metrics}
\label{sec:background:additive}

Let $S_{\bm{w}}(q, d)$ be a scoring function, parameterized by $\bm{w}$, that, for a given query $q$ and item $d \in D_q$ produces a real-valued ranking score, where $D_q$ is a set of candidate items for query $q$.
Let $\mathit{rel}(q, d) \in \{0, 1\}$ be the relevance of item $d$ to query $q$, where we assume binary relevance for simplicity.
We write $\mathit{rank}(d \mid q, D_q, S_{\bm{w}})$ for the rank of item $d \in D_q$ after sorting all items $D_q$ by their respective scores using the scoring function $S_{\bm{w}}$.
We consider the class of additive ranking metrics, as described in~\cite{agarwal2019general}:
\begin{equation}
\Delta ( S_{\bm{w}} \mid q ) = \sum_{d \in D_q} \lambda \rbrackets*{\mathit{rank}(d \mid q, D_q, S_{\bm{w}})} \cdot \mathit{rel}(q, d),
\label{eq:additiveranking}
\end{equation}
where $\lambda$ is a weighting function of the rank of an item that can capture different ranking metrics such as Average Relevant Rank, DCG, Precision@k and more~\cite{agarwal2019general}.
For simplicity we will assume $\lambda(x) = x$, but note that our findings hold for any convex, (sub)differentiable and monotonically increasing weighting function $\lambda$.
It is common practice to make Equation~\ref{eq:additiveranking} differentiable by upper bounding the $\mathit{rank}(d \mid q, D_q, S_{\bm{w}})$ term with a pairwise hinge loss~\cite{joachims2002optimizing,joachims2017unbiased,agarwal2019general}:
\begin{equation}
\begin{split}
\mathit{rank}&(d \mid q, D_q, S_{\bm{w}}) \leq {} \\
 &1 + \sum_{d' \in D_q} \max(0, 1 - (S_{\bm{w}}(q, d) - S_{\bm{w}}(q, d'))).
\end{split}
\label{eq:hingeloss}
\end{equation}
Now, suppose we are given a sample $Q$ of i.i.d. queries $q \sim P(q)$. The goal in \acf{LTR} is to learn a scoring function $S_{\bm{w}}$ that minimizes risk:
\begin{align}
&\mbox{}\hspace*{-2mm}
\arg\min_{\bm{w}} R(\bm{w}) = \arg\min_{\bm{w}} \int_{q} \Delta(S_{\bm{w}} \mid q) d P(q) \hspace*{-4mm}\mbox{} \nonumber \\
&\mbox{}\hspace*{-2mm}= \arg\min_{\bm{w}} \mathbb{E}_{Q}\sbrackets*{ \frac{1}{|Q|} \sum_{q \in Q} \sum_{d \in D_q} \lambda \rbrackets*{\mathit{rank}(d \mid q, D_q, S_{\bm{w}})} \cdot \mathit{rel}(q, d)}.
\hspace*{-2mm}\mbox{}
\label{eq:risk}
\end{align}
In most practical settings we cannot directly observe the relevance $\mathit{rel}(q, d)$, but only partial feedback in the form of clicks collected on rankings produced by a deployed production ranker.
As a result, we cannot directly minimize the empirical risk associated with Equation~\ref{eq:risk}.
Furthermore, unlike online \ac{LTR} approaches and bandit algorithms, we assume that we do not have any control over the deployed production ranker (i.e., we cannot perform \emph{interventions}).
In Counterfactual \ac{LTR} we instead focus on minimizing an unbiased estimate of the empirical risk using a historical click log.

\subsection{Counterfactual~Learning~to~Rank~with Biased Feedback}
\label{sec:background:cltr}
Suppose a deployed production ranker is collecting user interactions, i.e., clicks, as follows:
\begin{itemize}[leftmargin=*,nosep]
\item a user issues a query $q \sim P(q)$;
\item the user is presented with a ranking of the candidate items $D_{q}$ for the issued query; and
\item the user observes and clicks on an item $d \in D_{q}$ with probability $P(c(d) = 1)$, where $c(d) \in \{0, 1\}$ indicates a click on item $d$.
\end{itemize}
In line with existing work~\cite{craswell2008experimental,joachims2017unbiased} we assume that the examination hypothesis holds, i.e., that clicks can only occur on observed items.
In other words, a click on an item depends on the probability that the user observes the item \emph{and} decides to click on it:
\begin{equation}
P\rbrackets*{c(d) = 1} \stackrel{\mathrm{def}}{=} P\rbrackets*{o(d) = 1} \cdot P\rbrackets*{c(d) = 1 \mid o(d) = 1} ,
\end{equation}
where $o(d) \in \{0, 1\}$ indicates that item $d$ was observed by the user.
Finally, for any two observed items $d$ and $d'$ it is more likely that a user clicks on a relevant item than a non-relevant item. More formally:
\begin{equation}
\begin{split}
&\mathit{rel}(q, d) > \mathit{rel}(q, d')  \\
&\implies P\rbrackets*{c(d) = 1 \mid o(d) = 1} > P\rbrackets*{c(d') = 1 \mid o(d') = 1}.
\end{split}
\end{equation}
Under these assumptions a click does not necessarily indicate relevance, nor does a non-click necessarily indicate non-relevance.
In general $c(d) \neq  \mathit{rel}(q, d)$, and naively optimizing the empirical risk using clicks would be a suboptimal strategy~\cite{joachims2017unbiased}.
Instead, it is necessary to correct for the observation probabilities $P(o(d) = 1)$, often called the \emph{propensity}.
This motivates the use of \acf{IPS}, where inversely weighing the propensities of clicked items can debias the click data.

For our work we assume that the propensities $P(o(d) = 1)$ are known a priori.
In practice, the propensities are commonly estimated via A/B testing~\cite{wang2016learning} or through intervention harvesting~\cite{agarwal2019estimating,fang2019intervention}, usually under some assumption of a user behavior model such as the position-bias model~\cite{craswell2008experimental}.
Modeling and estimating these propensity scores falls outside the scope of this paper as our aim is to understand and improve the \emph{convergence rate} of \ac{IPS}-weighted optimization for \ac{LTR}.
We refer to existing work~\cite{ai2018unbiased,agarwal2019estimating,fang2019intervention,chandar2018estimating,carterette2018offline} for more detailed information about how the propensity scores can be modelled and/or estimated.

We now reach the main approach for Counterfactual \ac{LTR}, which is to apply \acf{IPS} to debias our optimization objective.
Suppose we are given a \emph{click log dataset} containing $n$ clicks:
\begin{equation}
\mathcal{D} = \cbrackets*{\rbrackets*{q_i, D_{q_i}, d_i, p_i}}_{i=1}^{n},
\end{equation}
where:
\begin{itemize}[leftmargin=*,nosep]
\item $q_i \sim P(q_i)$ is the issued query;
\item $D_{q_i}$ is the set of candidate items;
\item $d_i \in D_{q_i}$ is the clicked item (i.e., $c(d_i) = 1$); and
\item $p_i = P(o(d_i) = 1)$ is the propensity of item $d_i$.
\end{itemize}
Our goal now is to solve the following optimization problem:
\begin{flalign}
\arg\min_{\bm{w}} R_{\textit{IPS}}(\bm{w}) &= \arg\min_{\bm{w}} \frac{1}{n} \sum_{i=1}^{n} \frac{1}{p_i} \lambda(\textit{rank}(d_i \mid q_i, D_{q_i}, S_{\bm{w}})) \nonumber \\
&= \arg\min_{\bm{w}} \frac{1}{n} \sum_{i=1}^{n} \frac{1}{p_i} f_i(\bm{w}).
\label{eq:cfrisk}
\end{flalign}
It is easy to show that $R_{\textit{IPS}}(\bm{w})$ is an unbiased estimate of $R(\bm{w})$, i.e., that $\mathbb{E}\sbrackets*{R_{\textit{IPS}}(\bm{w})} = \mathbb{E}\sbrackets*{R(\bm{w})}$, as long as $P(o(d) = 1) > 0$ for all $d$, using the proof of Section 4 in~\cite{joachims2017unbiased}.
We note that the minimization problem in Equation~\ref{eq:cfrisk} permits stochastic optimization and can be efficiently solved with a variety of \ac{SGD} approaches.
A common approach for counterfactual learning is to treat the inverse propensity scores as weights and multiply the gradient $\nabla f_{i}(\bm{w})$ with $\frac{1}{p_i}$ to obtain an unbiased gradient estimate~\cite{jagerman-2019-model,agarwal2019general,ai2018unbiased,joachims2018deep}.
We display this \ac{IPS}-weighted \ac{SGD} approach in Algorithm~\ref{alg:sgd:ipsweighted}.

\begin{algorithm}
\caption{\acs{IPS}-weighted \acf{SGD}}
\label{alg:sgd:ipsweighted}
\begin{algorithmic}
\State $\bm{w}_1 = \bm{0}$
\For{$t \gets 1, \ldots, T$}
\State $i_t \sim \mathit{Uniform}(0, n)$ \Comment{sample $i_t$ from uniform distribution}
\State $\bm{g}_t = \frac{1}{p_{i_t}} \nabla f_{i_t}(\bm{w}_{t})$ \Comment{compute IPS-weighted gradient}
\State $\bm{w}_{t+1} = \bm{w}_{t} - \eta_t \bm{g}_t$ \Comment{SGD update step}
\EndFor
\State $\bm{\bar{w}} = \frac{1}{T} \sum_{t=1}^{T} \bm{w}_{t}$
\State \Return $\bm{\bar{w}}$
\end{algorithmic}
\end{algorithm}

\noindent%
Scaling the gradients with \ac{IPS} weights, as is done in Algorithm~\ref{alg:sgd:ipsweighted}, is a common technique for unbiased \ac{LTR}~\cite{bendersky2017learning,jagerman-2019-model,agarwal2019general,ai2018unbiased,wang2018position}.
It is known that high-variance gradients can cause poor convergence for general \ac{SGD}-style algorithms~\cite{shalev2014understanding}.
However, to the best of our knowledge, the exact impact of \ac{IPS} weights on the convergence rate of \ac{SGD} algorithms for \ac{LTR} has not been studied.
In other words, the nature of the relationship between the \ac{IPS} weights and the convergence rate of \ac{SGD} algorithms is an open problem.
To address this problem, we will analyze Algorithm~\ref{alg:sgd:ipsweighted} and provide results describing the relationship between \ac{IPS} weights and the convergence rate in the next section.


\section{Convergence of \ac{IPS}-weighted \ac{SGD}}
\label{sec:analysis}

In this section our aim is to better understand the convergence rate of Algorithm~\ref{alg:sgd:ipsweighted} and analyze the impact of \ac{IPS} weights on the convergence rate.

\subsection{Convergence Rate Analysis}
\label{sec:background:convergence}

Let $R_{\textit{IPS}}(\bm{w}) = \frac{1}{n} \sum_{i=1}^{n} \frac{1}{p_i} f_{i}(\bm{w})$ be the function that we wish to minimize using Algorithm~\ref{alg:sgd:ipsweighted}.
We assume that each $f_{i}(\bm{w})$ is convex and, consequently, that $R_{\textit{IPS}}$ is convex since each $\frac{1}{p_{i}} > 0$.
In \ac{LTR} this is a reasonable assumption because the loss can often be convex, for example the pairwise hinge loss formulation presented in Equation~\ref{eq:hingeloss} is convex.
Furthermore, the rank weighting functions $\lambda(\cdot)$ presented in Section~\ref{sec:background:additive} are convex for, e.g., the average relevant rank or DCG weighting schemes.
Additionally, we assume $R_{\textit{IPS}}$ is minimized at some point $\bm{w}^* \in \arg\min_{\bm{w} : \norm{\bm{w}} \leq B} R_{\textit{IPS}}(\bm{w})$.
We denote with $M$ the maximum \ac{IPS} weight in the dataset: $M = \max_i \frac{1}{p_i}$.
As in previous analyses of regular SGD~\cite{shalev2014understanding}, the goal of our analysis is to bound the suboptimality of the solution produced by Algorithm~\ref{alg:sgd:ipsweighted}:
\begin{equation}
\expectation{R_{\textit{IPS}}(\bm{\bar{w}}) - R_{\textit{IPS}}(\bm{w}^{*})
}.
\end{equation}

\begin{theorem}
\label{theorem:ipsweights}
Let $f_{i}(\bm{w})$ be a convex function for each $i$ and let \ $\bm{w}^*$ be a minimizer of $R_{\textit{IPS}}(\bm{w}) = \frac{1}{n}\sum_{i=1}^{n} \frac{1}{p_i} f_{i}(\bm{w})$ such that $\norm{\bm{w}^*} \leq B$.
Assume that  $\norm{\nabla f_{i_t}(\bm{w}_{t})} \leq G$ for all $t$ and let $\bar{\bm{w}}$ be the solution produced by running Algorithm~\ref{alg:sgd:ipsweighted} for $T$ iterations with $\eta = \sqrt{\frac{B^2}{(MG)^2 T}}$. Then:
\begin{equation}
\expectation{R_{\textit{IPS}}(\bm{\bar{w}}) - R_{\textit{IPS}}(\bm{w}^{*})} \leq \frac{B(MG)}{\sqrt{T}}.
\end{equation}
\end{theorem}
\begin{proof}
\vspace{-0.1cm}
This convergence rate proof is a variant of the proof of Theorem 14.8 from~\cite{shalev2014understanding}, where we use \ac{IPS}-weighted gradients $\bm{g}_t$ and construct a bound on the gradient variance in Equation~\ref{eq:gradientvariance}.
Since our notation deviates slightly from the notation in~\cite{shalev2014understanding}, we include the full proof here for clarity.
Denote with $i_{1:T}$ the sequence of random indices $i_1, \ldots, i_T$, then:
\begin{subequations}
\begin{flalign}
&\mathbb{E}_{i_{1:T}}\sbrackets*{R_{\textit{IPS}}(\bar{\bm{w}}) - R_{\textit{IPS}}(\bm{w^*})} \nonumber  \\
&= \mathbb{E}_{i_{1:T}}\sbrackets*{R_{\textit{IPS}}\rbrackets*{\frac{1}{T}\sum_{t=1}^{T}\bm{w}_{t}} - R_{\textit{IPS}}(\bm{w^*})}  \label{eq:thm1:1:definition} \\
&\leq \mathbb{E}_{i_{1:T}}\sbrackets*{\frac{1}{T} \sum_{t=1}^{T} R_{\textit{IPS}}(\bm{w}_{t}) - R_{\textit{IPS}}(\bm{w^*})} \label{eq:thm1:1:jensens} \\
&= \frac{1}{T} \sum_{t=1}^{T} \mathbb{E}_{i_{1:T}}\sbrackets*{R_{\textit{IPS}}(\bm{w}_{t}) - R_{\textit{IPS}}(\bm{w^*})}. \label{eq:thm1:1:linearity}
\end{flalign}
\end{subequations}
Here, \eqref{eq:thm1:1:definition} follows from the definition of $\bar{\bm{w}}$,  \eqref{eq:thm1:1:jensens} is due to Jensen's inequality~\citep{rudin-1987-real}, and, finally, \eqref{eq:thm1:1:linearity} is obtained by applying linearity of expectation.
Since $\bm{w}_{t}$ depends only on the indices $i_{1:t-1}$ we get:
\begin{flalign}
&\frac{1}{T}\sum_{t=1}^T \mathbb{E}_{i_{1:T}}\sbrackets*{R_{\textit{IPS}}(\bm{w}_{t}) - R_{\textit{IPS}}(\bm{w}^*)} \nonumber \\
&=\frac{1}{T}\sum_{t=1}^T \mathbb{E}_{i_{1:t-1}}\sbrackets*{R_{\textit{IPS}}(\bm{w}_{t}) - R_{\textit{IPS}}(\bm{w}^*)}.
\label{eq:iterates}
\end{flalign}
Once the indices $i_{1:t-1}$ are known, the value of \  $\bm{w}_{t}$ is no longer ran\-dom.
Furthermore, since each $i_t$ is uniformly sampled from the dataset, it follows that $\bm{g}_t$ is an unbiased estimate of $\nabla R_{\textit{IPS}}(\bm{w}_{t})$:
\begin{equation}
\mathbb{E}_{i_t}\sbrackets*{\bm{g}_t \mid i_{1:t-1}}
= \mathbb{E}_{i_t}\sbrackets*{\bm{g}_t \mid \bm{w}_{t}}
= \nabla R_{\textit{IPS}}(\bm{w}_{t}).
\label{eq:unbiasedgrad}
\end{equation}
Continuing from Equation~\ref{eq:iterates}, we have:
\begin{subequations}
\begin{flalign}
&\frac{1}{T}\sum_{t=1}^{T} \mathbb{E}_{i_{1:t-1}}\sbrackets*{R_{\textit{IPS}}(\bm{w}_{t}) - R_{\textit{IPS}}(\bm{w}^*)} \nonumber \\
&\leq \frac{1}{T}\sum_{t=1}^{T} \mathbb{E}_{i_{1:t-1}}\sbrackets*{\iprod*{ \bm{w}_{t} - \bm{w}^*, \nabla R_{\textit{IPS}}(\bm{w}_{t}) }} \label{eq:convexity:definition} \\
&= \frac{1}{T}\sum_{t=1}^{T} \mathbb{E}_{i_{1:t-1}}\sbrackets*{\iprod*{ \bm{w}_{t} - \bm{w}^*, \mathbb{E}_{i_t} \sbrackets*{\bm{g}_t \mid i_{1:t-1}} }} \label{eq:convexity:unbiased} \\
&= \frac{1}{T}\sum_{t=1}^{T} \mathbb{E}_{i_{1:t-1}}  \mathbb{E}_{i_{1:t}}\sbrackets*{\iprod*{\bm{w}_{t} - \bm{w}^*, \bm{g}_t} \mid i_{1:t-1} } \label{eq:convexity:linearity} \\
&= \frac{1}{T}\sum_{t=1}^{T} \mathbb{E}_{i_{1:t}} \sbrackets*{\iprod*{ \bm{w}_{t} - \bm{w}^*, \bm{g}_t }}, \label{eq:convexity:lawoftotal}
\end{flalign}
\end{subequations}
where \eqref{eq:convexity:definition} follows from the convexity of $R_{\textit{IPS}}$, \eqref{eq:convexity:unbiased} can be obtained by using Equation~\ref{eq:unbiasedgrad} ($\nabla R_{\textit{IPS}}(\bm{w}_t) = \mathbb{E}\sbrackets*{\bm{g}_t \mid i_{1:t-1}}$), \eqref{eq:convexity:linearity} is due to the linearity of expectation, and, finally, \eqref{eq:convexity:lawoftotal} is obtained by applying the law of total expectation.
We can now use Lemma 14.1 from~\cite{shalev2014understanding} since Equation~\ref{eq:convexity:lawoftotal} is of the required form and obtain:
\begin{flalign}
&\frac{1}{T}\sum_{t=1}^{T} \mathbb{E}_{i_{1:t}} \sbrackets*{\iprod*{ \bm{w}_{t} - \bm{w}^*, \bm{g}_t }} \leq \frac{1}{T} \rbrackets*{ \frac{\norm*{\bm{w}^*}^2}{2\eta} + \frac{\eta}{2} \sum_{t=1}^T \norm{\bm{g}_t}^2 }. \label{eq:lemma14:1}
\end{flalign}
Here is where we deviate from the standard proof of convergence for \ac{SGD} and upper bound the \ac{IPS}-weighted gradients $\bm{g}_t$ as follows:
\begin{equation}
\norm*{\bm{g}_t} = \norm*{\frac{1}{p_{i_t}} \nabla f_{i_t}(\bm{w}_{t})} = \frac{1}{p_{i_t}}\norm*{\nabla f_{i_t}(\bm{w}_{t})} \leq M G.
\label{eq:gradientvariance}
\end{equation}
Next, we can plug this upper bound on $\norm*{\bm{g}_t}$ into Equation~\ref{eq:lemma14:1}, and use the assumption that $\norm*{\bm{w}^*} \leq B$ to obtain:
\begin{flalign}
\frac{1}{T} \rbrackets*{\frac{\norm*{\bm{w}^*}^2}{2\eta} + \frac{\eta}{2} \sum_{t=1}^T \norm{\bm{g}_t}^2} \leq \frac{1}{T} \rbrackets*{\frac{B^2}{2\eta} + \frac{\eta}{2} T (MG)^2}.
\end{flalign}
Finally, by plugging in $\eta = \sqrt{\frac{B^2}{(MG)^2 T}}$ and applying algebraic manipulations we obtain:
\begin{subequations}
\begin{flalign}
&\frac{1}{T} \rbrackets*{\frac{B^2}{2\eta} + \frac{\eta}{2} T (MG)^2} \nonumber \\
&= \frac{1}{T} \rbrackets*{\frac{B^2}{2\sqrt{\frac{B^2}{(MG)^2T}}} + \frac{\sqrt{\frac{B^2}{(MG)^2T}}}{2} T (MG)^2} \\
&= \frac{1}{T} \rbrackets*{\frac{B (MG) \sqrt{T}}{2} + \frac{B(MG)\sqrt{T}}{2}} \\
& = \frac{B(MG) \sqrt{T}}{T} = \frac{B(MG)}{\sqrt{T}}.\qedhere
\end{flalign}
\end{subequations}
\end{proof}

\noindent%
Theorem~\ref{theorem:ipsweights} combines several important quantities that determine how fast Algorithm~\ref{alg:sgd:ipsweighted} will converge.
First, we have $1/\sqrt{T}$, which indicates that, as the number of iterations $T$ grows, the solution $\bm{\bar{w}}$ gets closer to the optimal solution $\bm{w}^*$.
Second, we have $B$, which tells us how far away $\bm{w}^*$ is from the starting point $\bm{0}$.
Clearly, for large $B$, the optimum $\bm{w}^*$ is far away and we require more iterations to converge.
Finally, we have the gradient variance term $(MG)$.
When the gradients have potentially large variance, we need to correspondingly set a small learning rate to prevent divergent behavior and, as a result, it takes longer to converge.
As a consequence of Theorem~\ref{theorem:ipsweights}, it is clear that in order to achieve an error of at most $\epsilon$, it suffices to run Algorithm~\ref{alg:sgd:ipsweighted} for $T$ iterations where:
\begin{equation}
T \geq \frac{B^2 (MG)^2}{\epsilon^2}.
\label{eq:nriterations}
\end{equation}

\subsection{Discussion}
\label{sec:background:improvements}
The key insight that our analysis provides is that, for \ac{IPS}-weighted \ac{SGD} algorithms, the number of iterations required to achieve an $\epsilon$-optimal solution grows with a factor $(MG)^2$.
We note that there are known variations of \ac{SGD} that can improve the convergence rate presented in Theorem~\ref{theorem:ipsweights} from $\BigO{1/\sqrt{T}}$ to $\BigO{1/T}$, for example see \cite{shamir2013stochastic,hazan2007logarithmic,rakhlin2011making,hazan2014beyond}.
However, their analyses are considerably more complex and do not remove the dependency on $\norm*{\bm{g}_t}^2$, which for the \ac{IPS}-weighted \ac{SGD} case remains upper bounded by $(MG)^2$.
This means that despite having faster convergence in terms of $T$, these methods do not improve the slowdown introduced by the \ac{IPS}-weights.

Moreover, \citet{agarwal2009information} show that, for strongly convex and Lipschitz smooth functions, the term $\norm*{\bm{g}_t}^2$, and consequently the term $(MG)^2$ in Equation~\ref{eq:nriterations}, cannot be improved for \emph{any SGD algorithm} (for sufficiently large $T$).
In practice, the value of $M$ can be very large, for example in cases with small propensities $p_i$ (e.g., in situations with a significant amount of position bias such as when the candidate set $D_q$ is very large).
The above facts lead us to hypothesize that, as long as the gradients are scaled with \ac{IPS} weights, the convergence rate is severely slowed by the magnitude of the \ac{IPS} weights.
Our experiments in Section~\ref{sec:results} support this hypothesis.


\section{Improved Convergence with Weighted Sampling}
\label{sec:method}

\subsection{\OurMethod{}:~\acs{SGD} with \acs{IPS}-proportional Sampling}
\label{sec:method:countersample}

Theorem~\ref{theorem:ipsweights} shows that the convergence of \ac{IPS}-weighted \ac{SGD} is slowed by a factor $M^2$.
Moreover, as described in Section~\ref{sec:background:improvements}, as long as gradients $\bm{g}_t$ are scaled by IPS weights, this dependency on $M$ cannot be improved~\cite{agarwal2009information}.
Clearly, for situations where $M$ is large, this can lead to slow convergence and make learning inefficient.

To overcome this problem we propose a sampling-based \ac{SGD} strategy.
The key idea is to debias our optimization objective via \emph{sampling} instead of \emph{weighting}.
As we will prove below, this sampling-based approach similarly guarantees unbiasedness of the optimization objective but has a better convergence rate.
We call our approach \OurMethod{} and it is displayed in Algorithm~\ref{alg:sgd:ipssampling}.

\begin{algorithm}
\caption{\OurMethod: \ac{SGD} with \ac{IPS}-proportional sampling}
\label{alg:sgd:ipssampling}
\begin{algorithmic}
\State $\bm{w}_{1}\gets \bm{0}$
\State $\bar{M} \gets \frac{1}{n} \sum_{i=1}^{n} \frac{1}{p_i}$
\For{$t \gets 1, \ldots, T$}
\State $i_t \sim P(i_t \mid \mathcal{D})$ \Comment{sample $i_t$ according to Equation~\ref{eq:sampling-probability}}
\State $\bm{g}_t = \bar{M}~\nabla f_{\lambda, i_t}(\bm{w}_{t})$ \Comment{compute gradient}
\State $\bm{w}_{t+1} = \bm{w}_{t} - \eta \bm{g}_t$ \Comment{\acs{SGD} update step}
\EndFor
\State $\bm{\bar{w}} \leftarrow \frac{1}{T} \sum_{t=1}^{T} \bm{w}_{t}$
\State \Return $\bm{\bar{w}}$
\end{algorithmic}
\end{algorithm}

\noindent%
\OurMethod{} functions as follows.
First, we assign each data point $i$ in our dataset $\mathcal{D}$ the following probability of being sampled:
\begin{equation}
P(i \mid \mathcal{D}) = \frac{\frac{1}{p_i}}{\sum_{j=1}^{n}\frac{1}{p_j}}.
\label{eq:sampling-probability}
\end{equation}
We then proceed exactly like regular \ac{SGD}, where instead of sampling datapoints uniformly from the dataset, they are sampled using Equation~\ref{eq:sampling-probability}.
Furthermore, the algorithm does not scale the gradients by the \ac{IPS}-weights, but by a constant factor:
\begin{equation}
\bar{M} = \frac{1}{n} \sum_{i=1}^{n} \frac{1}{p_i},
\end{equation}
which is necessary to guarantee unbiasedness (see Section~\ref{sec:method:unbiased}).

In practice, one would tune the learning rate $\eta$ in Algorithm~\ref{alg:sgd:ipssampling}, making the inclusion of $\bar{M}$ unnecessary as it merely scales the gradients by a constant, which can be offset by the particular $\eta$ chosen.
Therefore, for practical implementations, it is not necessary to include this constant.
We include $\bar{M}$ here for the purposes of guaranteeing unbiasedness and analyzing the convergence rate.

\subsection{Unbiasedness}
\label{sec:method:unbiased}
To show that Algorithm~\ref{alg:sgd:ipssampling} minimizes the unbiased objective $R_{\textit{IPS}}(\bm{w})$, it is sufficient to show that, in expectation, the gradient $\bm{g}_t$ is an unbiased estimate of $\nabla R_{\textit{IPS}}(\bm{w_{t}})$ for any $t$.
\begin{theorem}
\label{thm:unbiasedness}
Let $R_{\textit{IPS}}(\bm{w})$ be the function to be optimized and let $\bm{g}_t$ be the gradient at time $t$ as computed by Algorithm~\ref{alg:sgd:ipssampling}, then:
\begin{equation}
\mathbb{E}_{i_t}\sbrackets*{\bm{g}_t \mid \bm{w}_{t}} = \nabla R_{\textit{IPS}}(\bm{w_{t}}).
\end{equation}
\end{theorem}
\begin{proof}
The proof uses the definition of $\bm{g}_t$ and linearity of expectation (\ref{eq:unbiased:linearity}) and the definition of expectation (\ref{eq:unbiased:definition}) where we use Equation~\ref{eq:sampling-probability} as the sampling probability. Using the definition of $\bar{M}$ and algebraic manipulations completes the proof (\eqref{eq:unbiased:algebraic1} and \eqref{eq:unbiased:algebraic2}):
\begin{subequations}
\begin{flalign}
\mathbb{E}_{i_t}\sbrackets*{\bm{g}_t \mid \bm{w}_{t}} =&~\mathbb{E}\sbrackets*{\bar{M}~\nabla f_{i_t}(\bm{w}_{t})} =~\bar{M}~\mathbb{E}\sbrackets*{\nabla f_{i_t}(\bm{w}_{t})} \label{eq:unbiased:linearity} \\
=&~\bar{M}\sum_{i=1}^{n} P(i \mid \mathcal{D}) \nabla f_{i}(\bm{w}_{t}) \label{eq:unbiased:definition} \\
=&~\frac{1}{n} \rbrackets*{\sum_{i=1}^{n} \frac{1}{p_i}} \sum_{i=1}^{n} \frac{\frac{1}{p_i}}{\sum_{j=1}^{n} \frac{1}{p_j}}   \nabla f_{i}(\bm{w}_{t}) \label{eq:unbiased:algebraic1} \\
=&~\frac{1}{n} \sum_{i=1}^{n} \frac{1}{p_i} \nabla f_{i}(\bm{w}_{t}) = \nabla R_{\textit{IPS}}(\bm{w}_{t}).
\qedhere
\label{eq:unbiased:algebraic2}
\end{flalign}
\end{subequations}
\end{proof}

\subsection{Convergence Rate}
\label{sec:method:convergence}
We wish to understand the convergence rate of the proposed method \OurMethod{} (Algorithm~\ref{alg:sgd:ipssampling}).
Similar to Section~\ref{sec:background:convergence}, the goal of our analysis is to bound the suboptimality of the solution produced by Algorithm~\ref{alg:sgd:ipssampling}:
\begin{equation}
\expectation{R_{\textit{IPS}}(\bm{\bar{w}}) - R_{\textit{IPS}}(\bm{w}^{*})}.
\end{equation}

\begin{theorem}
\label{theorem:ipssampling}
Let $f_{i}(\bm{w})$ be a convex function for each $i$ and let \ $\bm{w}^*$ be a minimizer of $R_{\textit{IPS}}(\bm{w}) = \frac{1}{n}\sum_{i=1}^{n} \frac{1}{p_i} f_{i}(\bm{w})$ such that $\norm{\bm{w}^*} \leq B$.
Assume that  $\norm{\nabla f_{i_t}(\bm{w}_{t})} \leq G$ for all $t$ and let \ $\bar{\bm{w}}$ be the solution produced by running Algorithm~\ref{alg:sgd:ipssampling} for $T$ iterations with $\eta = \sqrt{\frac{B^2}{(\bar{M}G)^2 T}}$. Then:
\begin{equation}
\expectation{R_{\textit{IPS}}(\bm{\bar{w}}) - R_{\textit{IPS}}(\bm{w}^{*})} \leq \frac{B(\bar{M}G)}{\sqrt{T}}.
\end{equation}
\end{theorem}
\begin{proof}
First, we note that in Algorithm~\ref{alg:sgd:ipssampling}, the gradients are bounded as follows:
\begin{equation}
\norm*{\bm{g}_t} = \norm*{\bar{M}~\nabla f_{i_t}(\bm{w}_t)} \leq \bar{M} G.
\label{eq:samplegradientvariance}
\end{equation}
Next, we follow the proof of Theorem~\ref{theorem:ipsweights}, replacing Equation~\ref{eq:gradientvariance} with Equation \ref{eq:samplegradientvariance}, using the result of Theorem~\ref{thm:unbiasedness} to ensure $\mathbb{E}_{i_t}\sbrackets*{\bm{g}_t \mid \bm{w}_{t}} = \nabla R_{\textit{IPS}}(\bm{w_{t}})$ and plugging in $\eta = \sqrt{\frac{B^2}{(\bar{M}G)^2 T}}$ to give us the desired result.
\end{proof}

\noindent%
As a result of Theorem~\ref{theorem:ipssampling}, we can conclude that \OurMethod{} provides significant advantages in terms of convergence rate over standard \ac{IPS} weighting.
Specifically, to obtain an error of at most $\epsilon$, it is sufficient to run \OurMethod{} for:
\begin{equation}
T \geq \frac{B^2 (\bar{M}G)^2}{\epsilon^2}
\end{equation}
iterations.
This is strictly better than the bound that was obtained for Algorithm~\ref{alg:sgd:ipsweighted} in nearly all cases.
The only case where the two methods have the same convergence rate is when $\bar{M} = M$, which can only happen when all the propensity scores are the same, i.e., when $p_i = p_j$ for all $i, j$.
However, this can only be the case when the click log itself is already unbiased, thus negating the need to do counterfactual learning in the first place.
Therefore, for any practical Counterfactual \ac{LTR} scenario, \OurMethod{} is strictly better in terms of convergence rate than naively scaling the gradients with \ac{IPS} weights.

\subsection{Efficiency}
\label{sec:method:efficiency}
Finally, despite the advantages of \OurMethod{} in terms of convergence rate, these benefits may not be useful if they come at the cost of worse computational complexity.
A straightforward but naive implementation for sampling from Equation~\ref{eq:sampling-probability} would result in a $\BigO{Tn}$ time complexity for Algorithm~\ref{alg:sgd:ipssampling}, which is significantly worse than the $\BigO{T}$ complexity obtained by standard SGD approaches such as Algorithm~\ref{alg:sgd:ipsweighted}.

Fortunately, sampling from Equation~\ref{eq:sampling-probability} can be done with an amortized $\BigO{1}$ cost using the alias method~\cite{walker1977efficient,walker1974new}.
To achieve this, there is a one time cost of constructing the alias table, done in $\BigO{n}$. Furthermore, we also need to compute the constant $\bar{M}$ which is also a one time operation of $\BigO{n}$.
We note that both of these steps can take place during data pre-processing and are for most practical implementations easily achieved (e.g., PyTorch~\cite{paszke2017automatic} and Tensorflow~\cite{abadi2016tensorflow} both support efficient sampling from a weighed multinomial distribution using the alias method).
Overall, this means that the complexity of \OurMethod{} is $\BigO{n + T}$, which is acceptable when $n < T$.

\subsection{Illustrative Example}
\begin{figure}
\includegraphics[clip,trim=5mm 6mm 5mm 0mm,width=\columnwidth]{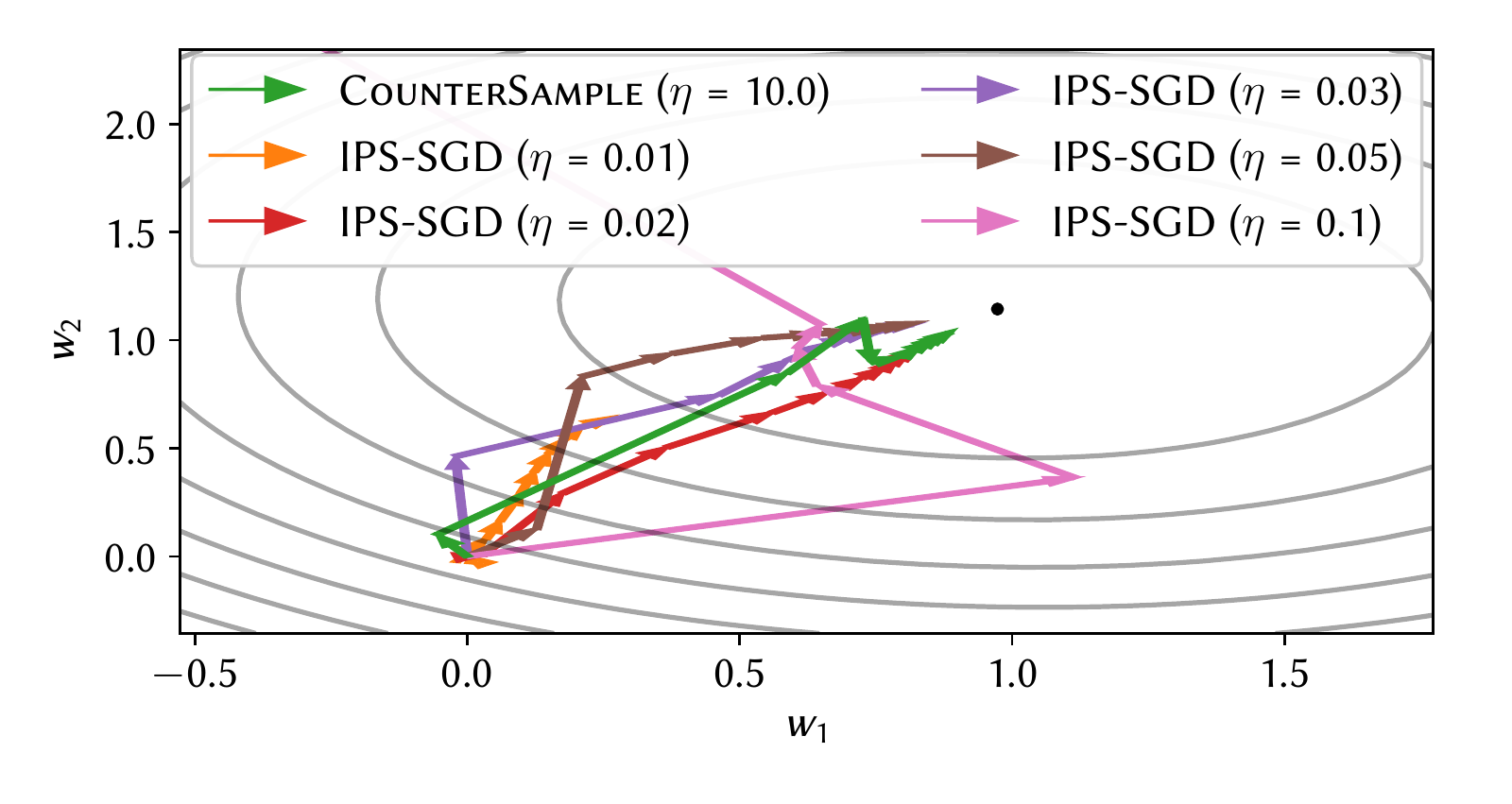}
\caption{Illustration of the convergence of \OurMethod{} versus \ac{IPS}-weighted \ac{SGD} on a synthetic learning example with two weights $w_1$ and $w_2$. The algorithms are run for $T = 50$ iterations. Best viewed in color.}
\Description[Figure shows toy example with 2 weights]{Figure shows toy example with 2 weights. In this example \OurMethod{} gets closer to the optimum after 50 iterations than \IPSSGD{} for varying learning rates.}
\label{fig:toy}
\end{figure}

To illustrate the difference in convergence rates between \OurMethod{} and standard \ac{IPS}-weighted \ac{SGD} we have created a simple toy example learning problem in Figure~\ref{fig:toy}.
We chose two optimal weights \ $\bm{w}^* = [w_1, w_2]$ and synthesize an IPS-weighted regression dataset.\footnote{The synthesized dataset comprises 50 training samples: $\{\bm{x}_1, \ldots, \bm{x}_{50}\}$ where each $\bm{x}_i = [x_{i,1}, x_{i,2}]$ and each $x_{i,j} \sim \mathcal{N}(0, 1)$. We choose $\bm{w}^* = [0.973, 1.144]$ and set targets $y_i = \iprod*{\bm{x}_i, \bm{w}^*}$. For each $\bm{x}_i$ we generate a propensity $p_i \sim \textit{Uniform}(0.05, 1.0)$ and use the IPS-weighted squared loss as our optimization objective: $R_{\textit{IPS}}(\bm{w}) = \frac{1}{50} \sum_{i=1}^{50} \frac{1}{p_i} \rbrackets*{\iprod*{\bm{x}_i, \bm{w}} - y}^2$.}
Notice that, for large learning rates, the IPS-weighted \ac{SGD} approach leads to unstable learning and diverges from the optimum.
The learning rate for \ac{IPS}-weighted \ac{SGD} needs to be small enough to ensure that training samples with large \ac{IPS} weights do not cause too large a step, possibly leading to divergent behavior.
As a result, it is necessary to reduce the learning rate to ensure stable learning, however doing so naturally increases the time until convergence for IPS-weighted SGD because samples that do not have extreme \ac{IPS} weights can only make small progress to the optimum.
On the other hand, \OurMethod{} can reliably handle large learning rates because the gradients are not scaled with potentially large \ac{IPS} weights.
Instead, \OurMethod{} samples training instances with high \ac{IPS} weights more frequently.
Overall, this leads to much faster convergence.


\section{Experimental setup}
\label{sec:setup}

Our experimental setup is aimed at assessing the convergence rate of \OurMethod{} and to answer the following research question:
\begingroup
\addtolength\leftmargini{-0.05in}
\begin{quote}
\emph{Does \OurMethod{}, a sampling-based \ac{SGD} approach, converge faster than \ac{IPS}-weighted \ac{SGD} for \ac{LTR}?}
\end{quote}
\endgroup

\noindent%
We use the standard experimental setup for Counterfactual \ac{LTR}, first described in~\cite{joachims2017unbiased}.
This means that we use a \emph{fully supervised} \ac{LTR} dataset and simulate a biased \emph{click log} according to a position-based user behavior model.

\subsection{Datasets}
\label{sec:setup:datasets}
\begin{table}
\caption{Datasets used for our experiments.}
\label{tbl:datasets}
\begin{tabular}{lrrr}
\toprule
Dataset & Queries & Avg. docs per query &  $\textit{rel}(q, d) = 0$ \\
\midrule
\Yahoo{}~\cite{chapelle2011yahoo} & 36,251 & 23 & 26\% \\
\Istella{}~\cite{lucchese2016post} & 33,118 & 103 & 89\% \\
\bottomrule
\end{tabular}
\end{table}
We use two supervised \ac{LTR} datasets in our experiments: \Yahoo{}~\cite{chapelle2011yahoo} and \Istella{}~\cite{lucchese2016post}.
We choose these two datasets as they complement each other in the number of items per query, which is large for \Istella{} and small for \Yahoo{}, and the sparsity of relevance feedback, which is high for \Istella{} and low for \Yahoo{} (see Table~\ref{tbl:datasets}).

Both \ac{LTR} datasets are collected on a large set of queries $\mathcal{Q} = \cbrackets*{q_1, \ldots, q_m}$.
For each query $q$ the dataset provides a set of candidate items $D_q$, where each item $d \in D_q$ is given in the form of a feature vector $x_{(q, d)}$ representing a query-item pair.
Furthermore, for each query $q$ the relevance grades $\mathit{rel}(q, d)$ are known for all $d \in D_q$.
The relevance grades are scaled from 0 to 4, where 0 indicates no relevance and 4 indicates highly relevant.
We note that this violates the binary relevance assumption made in Section~\ref{sec:background:additive}.
However, as we will see in Section~\ref{sec:setup:simulation}, during click simulation the relevance grades are reduced to binary form which is in line with existing experimental setups for Counterfactual \ac{LTR}~\cite{joachims2017unbiased}.

\subsection{Simulation Setup}
\label{sec:setup:simulation}
We simulate clicks using the setup of~\cite{joachims2017unbiased}.
In this setup, we repeatedly sample a query $q$ uniformly from the dataset.
The candidate items $D_q$ for the sampled query are then sorted by a scoring function $S_0$, called the \emph{logging policy} (see Section~\ref{sec:setup:loggingpolicy} for how $S_0$ is chosen).
The simulation introduces position bias: items that are highly ranked by $S_0$ have a higher probability of being observed and thus clicked.
Furthermore, the simulation has some noise: for every observed item, the probability of it being clicked is 1 if the item is relevant ($\textit{rel}(q, d) \in \{3, 4\}$) and 0.1 if the item is not relevant ($\textit{rel}(q, d) \in \{0, 1, 2\}$).
More formally, for every item $d \in D_q$, clicks are sampled from a Bernoulli distribution with probability:
\begin{equation}
\mbox{}\hspace*{-2mm}
P(c(d) = 1) \ {=} \ \left\{ \mbox{}\hspace*{-1mm}\begin{array}{ll}
P(o(d) \mid q, D_q, S_0) & \mbox{}\hspace*{-2mm}\text{if }\textit{rel}(q, d) \in \{3, 4\}, \\
P(o(d) \mid q, D_q, S_0) \cdot 0.1 & \mbox{}\hspace*{-2mm}\text{if }\textit{rel}(q, d) \in \{0, 1, 2\}, \hspace*{-3mm}\mbox{}\\
\end{array} \right.
\end{equation}
where
\begin{equation}
P(o(d) \mid q, D_q, S_0) = \rbrackets*{\frac{1}{\textit{rank}(d \mid q, D_q, S_0)}}^\gamma,
\end{equation}
and $\gamma \geq 0$ is a parameter controlling the severity of position bias.
The above formulation is identical to the setup used in~\cite{joachims2017unbiased}.
For all our experiments we simulate 1,000,000 clicks.
Unless otherwise specified, we use $\gamma = 1$ as the position bias parameter.
Table~\ref{tbl:propensities} shows the values of $M$ and $\bar{M}$ for the simulated clicks on each of the datasets.
We note that even under mild position bias ($\gamma = 0.5$), there is a significant difference between $M$ and $\bar{M}$.
The difference between these quantities becomes substantially larger as $\gamma$ increases.

Simulating clicks from supervised datasets has several advantages over using existing click logs.
First, by simulating clicks we can explicitly control the severity of position bias (by controlling the value of $\gamma$) and therefore test its impact in a controlled environment.
Second, we can evaluate the learned rankers on the \emph{true relevance labels} as they are provided by the supervised datasets.
This means that we do not have to resort to performance estimation techniques that may be unreliable.

\begin{table}
\caption{$\bar{M}$ and $M$ for varying levels of position bias ($\gamma$).}
\label{tbl:propensities}
\begin{tabular}{l@{\hspace{6mm}}rrrrr}
\toprule
Position bias ($\gamma$): & 0.5 & 0.75 & 1.0 & 1.25 & 1.5 \\
\midrule
\Yahoo{}: $\bar{M}$ & 3.14& 5.12& 7.92& 11.71& 16.66 \\
\Yahoo{}: $M$ & 11.79& 40.70& 129.00& 388.91& 1265.55 \\
\midrule
\Istella{}: $\bar{M}$ & 4.21& 7.42& 12.02& 18.12& 25.94 \\
\Istella{}: $M$ & 13.04& 48.53& 177.00& 645.60& 2081.04 \\
\bottomrule
\end{tabular}
\end{table}

\subsection{Choice of Logging Policy}
\label{sec:setup:loggingpolicy}
We need to build a logging policy $S_0$ that can be used to rank items for our click simulation.
A good candidate logging policy is one that can produce rankings of sufficient quality to generate a useful number of relevant clicks, but not perfectly optimal so that learning can still occur.
To do so we train a linear ranker with full supervision (using the pairwise hinge loss formulation of Equation~\ref{eq:risk}) on 0.1\% of the queries for each of the datasets.
Building a logging policy in this manner represents a realistic deployment scenario: practitioners of \ac{LTR} systems would typically train a ranker on a small amount of manually annotated data before deploying it to collect a large amount of click data.

\subsection{Evaluation}
\label{sec:setup:evaluation}
To measure the performance of the rankers learned by the various algorithms, we use nDCG@10~\cite{jarvelin2002cumulated} on held-out test data.
We denote with $\textit{nDCG@10}(\bm{w})$, the average nDCG@10 on held-out test data when items are ranked using the scoring function $S_{\bm{w}}$.

We are interested in measuring the \emph{convergence rate} of the learning algorithms.
To do so, we measure average regret in terms of nDCG@10 (with respect to the optimal model \ $\bm{w}^*$):
\begin{flalign}
&\textit{Regret}(T) = \frac{1}{T}\sum_{t=1}^{T} \left(\textit{nDCG@10}(\bm{w}^*) - \textit{nDCG@10}\rbrackets*{\bar{\bm{w}}_t}\right),
\label{eq:regret}
\end{flalign}
where $\bar{\bm{w}}_t = \frac{1}{t}\sum_{t^\prime=1}^t \bm{w}_{t^\prime}$ is the learned model after $t$ iterations.
To obtain the gold standard $\bm{w}^*$ we train a linear ranker with full supervision (using the relevance labels of the \ac{LTR} dataset).

Measuring regret should help us confirm our theoretical results about convergence rates since lower values indicate faster convergence to the optimal solution $\bm{w}^*$.
For each of our results we consider statistical significance with a $t$-test ($p < 0.01$).

\subsection{Methods to Compare}
In our experiments we compare the following methods:
\begin{itemize}[leftmargin=*,nosep]
\item \OurMethod{}: our sample-based method (Algorithm~\ref{alg:sgd:ipssampling});
\item \IPSSGD{}: \ac{IPS}-weighted \ac{SGD} (Algorithm~\ref{alg:sgd:ipsweighted})~\cite{joachims2017unbiased,agarwal2019general}; and 
\item \BiasedSGD{}: naive \ac{SGD} without any propensity weighting.
\end{itemize}
We use a linear scoring function for our experiments, as this more closely matches the convexity assumptions made in our analysis:
\begin{equation}
S_{\bm{w}}(q, d) = \iprod*{\bm{w}, x_{q, d}}.
\end{equation}
For each experiment and each method we tune the learning rate $\eta$ to minimize regret (see Section~\ref{sec:setup:evaluation}) on held-out validation data, where we try the following values of $\eta$:
\begin{equation}
\eta \in \{1 \times 10^{-10}, 3 \times 10^{-10}, 1 \times 10^{-9}, \ldots, 1 \times 10^{0}, 3 \times 10^{0}\}.
\end{equation}


\section{Experimental Results}
\label{sec:results}

\subsection{Effect of Optimizer}
\begin{figure*}
\includegraphics[width=\textwidth]{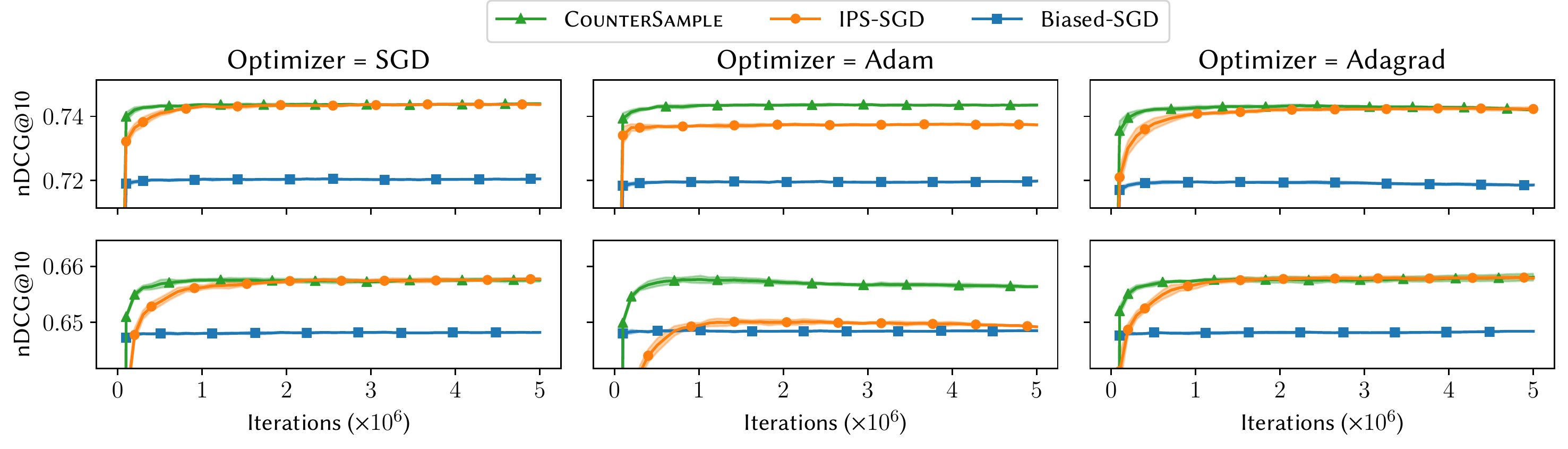}
\caption{The learning performance on held-out test data for different optimizers (top row is \Yahoo{}, bottom row is \Istella{}). \OurMethod{} is significantly faster to converge in all scenarios. For the Adam optimizer, the \ac{IPS}-weighted \ac{SGD} approach produces a suboptimal performance.}
\Description[Figure shows learning curves for various optimizers]{Figure shows learning curves for various optimizers. In all cases \OurMethod{} converges faster than \IPSSGD{}. For the \textsc{Adam} optimizer, \OurMethod{} converges to a higher level of performance than \IPSSGD{}.}
\label{fig:optimizers}
\end{figure*}
\begin{table}
\caption{Average regret ($\times 100$) for different optimizers. Smaller values indicate faster convergence. Statistically significantly lower and higher regret compared to \ac{IPS}-\ac{SGD} is denoted with $\triangledown$ and $\triangle$ respectively.}
\label{tbl:optimizers}
\begin{tabular}{l@{\hspace{6mm}}rrr}
\toprule
Optimizer: & SGD & \textsc{Adam} & \textsc{Adagrad} \\
\midrule
\Yahoo{} \\
\quad\BiasedSGD{} & 2.64\rlap{$^{\triangle}$} & 2.72\rlap{$^{\triangle}$} & 2.76\rlap{$^{\triangle}$} \\
\quad\IPSSGD{} & 0.41 & 0.97 & 0.64 \\
\quad\OurMethod{} & 0.33\rlap{$^{\triangledown}$} & 0.35\rlap{$^{\triangledown}$} & 0.44\rlap{$^{\triangledown}$} \\
\midrule
\Istella{} \\
\quad\BiasedSGD{} & 2.07\rlap{$^{\triangle}$} & 2.04\rlap{$^{\triangledown}$} & 2.06\rlap{$^{\triangle}$} \\
\quad\IPSSGD{} & 1.33 & 2.16 & 1.24 \\
\quad\OurMethod{} & 1.19\rlap{$^{\triangledown}$} & 1.24\rlap{$^{\triangledown}$} & 1.15\rlap{$^{\triangledown}$} \\
\bottomrule
\end{tabular}

\end{table}
First, we investigate the impact of the optimizer on the convergence rate of the different Counterfactual \ac{LTR} approaches.
We consider three commonly used optimization methods: Regular \ac{SGD}, \textsc{Adam}~\cite{kingma2014adam} and \textsc{Adagrad}~\cite{duchi2011adaptive}.
We apply these methods by replacing the update rule in Algorithms~\ref{alg:sgd:ipsweighted} and~\ref{alg:sgd:ipssampling} with either the update rule from \textsc{Adam} or \textsc{Adagrad}.

In Figure~\ref{fig:optimizers} we plot the learning curves on held-out test data for both the \Yahoo{} and \Istella{} dataset.
Interestingly, \IPSSGD{} does not work well with \textsc{Adam}, converging to a suboptimal solution, whereas \OurMethod{} is able to converge to a much higher level of performance.
This result is surprising as both \OurMethod{} and \IPSSGD{} optimize the same unbiased objective.
Recent work has shown that \textsc{Adam} is not guaranteed to converge to the optimal solution for some convex optimization problems and this may in part explain the behavior we observe here~\cite{reddi2019convergence}.
Regardless, we observe that in all cases \OurMethod{} converges significantly faster than \IPSSGD{}.
We note that the naive \BiasedSGD{} approach converges to a lower level of performance, which is as expected since it ignores the impact of position bias.
We confirm these findings by reporting the average regret in Table~\ref{tbl:optimizers}, observing a significantly lower regret for \OurMethod{} than \IPSSGD{}.

Our results indicate that \OurMethod{} is superior to \ac{IPS}-weighting across all optimizers.
We find that regular \ac{SGD} outperforms other optimizers in the majority of cases.
A possible reason for this behavior is that our scoring function is linear.
Optimizers such as \textsc{Adam} and \textsc{Adagrad} may not provide significant benefits over \ac{SGD} when applied to linear functions as opposed to non-linear functions (e.g., deep neural networks).
We leave studying other scoring functions such as neural networks as future work.

\subsection{Impact of Batch Size}
\begin{figure*}
\includegraphics[width=\textwidth]{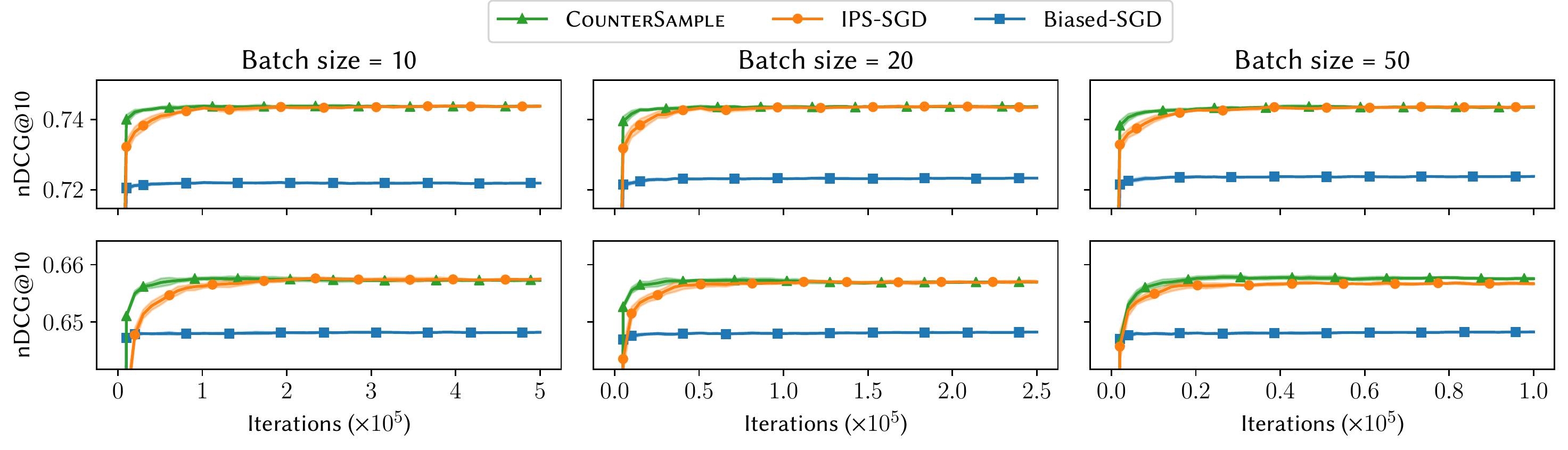}
\caption{The learning performance on held-out test data for different batch sizes (top row is \Yahoo{}, bottom row is \Istella{}). \OurMethod{} is faster to converge in all scenarios, however the differences are less pronounced for larger batch sizes.}
\Description[Figure shows learning curves for various batch sizes]{Figure shows learning curves for various batch sizes. In all cases \OurMethod{} converges faster than \IPSSGD{}. The effects are less pronounced when the batch sizes get larger.}
\label{fig:batchsizes}
\end{figure*}
\begin{table}
\caption{Average regret ($\times 100$) for different batch sizes. Statistical significance is denoted the same as Table~\ref{tbl:optimizers}.}
\label{tbl:batchsizes}
\begin{tabular}{l@{\hspace{6mm}}rrr}
\toprule
Batch size: & 10 & 20 & 50 \\
\midrule
\Yahoo{} \\
\quad\BiasedSGD{} & 2.49\rlap{$^{\triangle}$} & 2.36\rlap{$^{\triangle}$} & 2.31\rlap{$^{\triangle}$} \\
\quad\IPSSGD{} & 0.41 & 0.42 & 0.44 \\
\quad\OurMethod{} & 0.33\rlap{$^{\triangledown}$} & 0.34\rlap{$^{\triangledown}$} & 0.37\rlap{$^{\triangledown}$} \\
\midrule
\Istella{} \\
\quad\BiasedSGD{} & 2.07\rlap{$^{\triangle}$} & 2.08\rlap{$^{\triangle}$} & 2.07\rlap{$^{\triangle}$} \\
\quad\IPSSGD{} & 1.34 & 1.31 & 1.32 \\
\quad\OurMethod{} & 1.20\rlap{$^{\triangledown}$} & 1.21\rlap{$^{\triangledown}$} & 1.21\rlap{$^{\triangledown}$} \\
\bottomrule
\end{tabular}

\end{table}

In this section we investigate the effect of the batch size.
We hypothesize that large batch sizes reduce the variance of individual update steps, as many gradients are averaged in a single update step, and as a result the convergence rate of \OurMethod{} and \ac{IPS}-\ac{SGD} should be comparable.
We try batch sizes 10, 20 and 50.

We plot the learning curves for varying batch sizes in Figure~\ref{fig:batchsizes}.
Once again we find that, unsurprisingly, \BiasedSGD{} converges to a suboptimal solution.
For both datasets we observe that \OurMethod{} is able to converge faster than \IPSSGD{}, regardless of the chosen batch size.
For the \Istella{} dataset, \OurMethod{} is able to converge to a slightly higher level of performance than \IPSSGD{} when using a batch size of 50.
The average regret in Table~\ref{tbl:batchsizes} suggests that the convergence rate of \OurMethod{} is not affected by batch size; it is able to converge faster than \IPSSGD{} in all cases.

\subsection{Severity of Position Bias}
\begin{figure*}
\includegraphics[width=\textwidth]{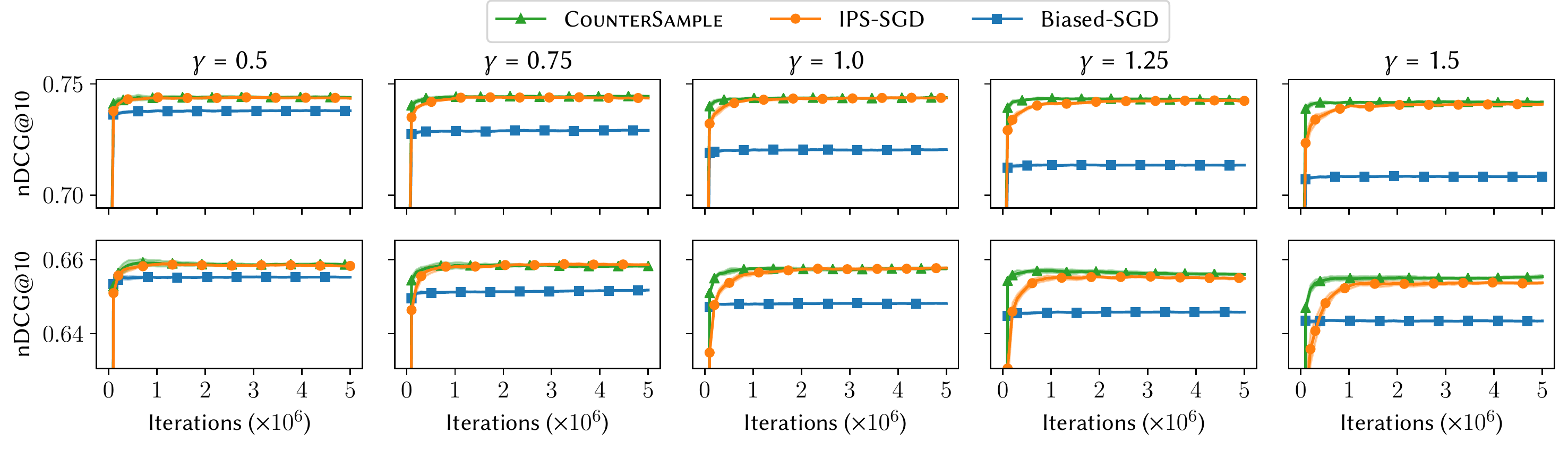}
\caption{The learning performance on held-out test data for varying levels of position bias (top row is \Yahoo{}, bottom row is \Istella{}). \OurMethod{}'s convergence rate is robust to larger values of $\gamma$, whereas \IPSSGD{} suffers when $\gamma$ is large.}
\label{fig:etas}
\Description[Figure shows learning curves for various levels of position bias $\gamma$]{Figure shows learning curves for various levels of position bias $\gamma$. As $\gamma$ increases, the learning curves for \OurMethod{} converges significantly faster than those for \IPSSGD{}.}
\end{figure*}
\begin{table}
\caption{Average regret ($\times 100$) for different levels of position bias $\gamma$. Statistical significance is denoted the same as Table~\ref{tbl:optimizers}.}
\label{tbl:etas}
\begin{tabular}{l@{\hspace{6mm}}rrrrr}
\toprule
Position bias ($\gamma$): & 0.5 & 0.75 & 1.0 & 1.25 & 1.5 \\
\midrule
\Yahoo{} \\
\quad\BiasedSGD{} & 0.89\rlap{$^{\triangle}$} & 1.78\rlap{$^{\triangle}$} & 2.64\rlap{$^{\triangle}$} & 3.32\rlap{$^{\triangle}$} & 3.83\rlap{$^{\triangle}$} \\
\quad\IPSSGD{} & 0.34 & 0.35 & 0.41 & 0.56 & 0.75 \\
\quad\OurMethod{} & 0.30 & 0.27\rlap{$^{\triangledown}$} & 0.33\rlap{$^{\triangledown}$} & 0.38\rlap{$^{\triangledown}$} & 0.51\rlap{$^{\triangledown}$} \\
\midrule
\Istella{} \\
\quad\BiasedSGD{} & 1.37\rlap{$^{\triangle}$} & 1.75\rlap{$^{\triangle}$} & 2.07\rlap{$^{\triangle}$} & 2.30\rlap{$^{\triangle}$} & 2.53\rlap{$^{\triangle}$} \\
\quad\IPSSGD{} & 1.09 & 1.12 & 1.33 & 1.53 & 1.79 \\
\quad\OurMethod{} & 1.05 & 1.08 & 1.19\rlap{$^{\triangledown}$} & 1.26\rlap{$^{\triangledown}$} & 1.44\rlap{$^{\triangledown}$} \\
\bottomrule
\end{tabular}

\end{table}

Finally, we look at the impact of position bias, controlled by the position bias parameter $\gamma$.
Position bias has an effect on the nature of the clicks collected and changes the distribution of propensity scores (see Table~\ref{tbl:propensities}).
For large $\gamma$, the propensity scores will be heavily skewed: the majority of observations and propensities will be on the top-ranked items while lower-ranked items are only very rarely observed and clicked, resulting in more extreme \ac{IPS} weights for those clicks.
Conversely, a small $\gamma$ makes the propensity scores more heavy tailed, generating more observations on lower ranked items and consequently more clicks on those items with less extreme \ac{IPS} weights.
We expect that, as we increase $\gamma$, \OurMethod{} should outperform \IPSSGD{} in terms of convergence rate since in this case $M \gg \bar{M}$.
Conversely, for smaller values of $\gamma$ we expect that the methods perform comparably.

Figure~\ref{fig:etas} provides learning curves for various levels of $\gamma$.
We observe that the performance of \BiasedSGD{} goes up as $\gamma$ goes down, which is in line with our expectations since smaller values of $\gamma$ result in less position bias.
In all cases \IPSSGD{} and \OurMethod{} perform strictly better than \BiasedSGD{}.
The convergence of \OurMethod{} is comparable to \IPSSGD{} for smaller $\gamma$, but as $\gamma$ grows, \OurMethod{} is significantly faster to converge than \IPSSGD{}.
This confirms our expectation that \OurMethod{} is able to reliably handle situations where $M \gg \bar{M}$, i.e. when there are more extreme \ac{IPS} weights.
Table~\ref{tbl:etas} confirms these findings in terms of average regret: for larger values of $\gamma$, \OurMethod{} is able to obtain significantly lower regret than competing approaches.

\subsection{Discussion}
\label{sec:results:discussion}
Finally, we reflect on the research question posed in Section~\ref{sec:setup}: \textit{Does \OurMethod{}, a sampling-based \ac{SGD} approach, converge faster than \ac{IPS}-weighted \ac{SGD} for \ac{LTR}?}
We answer our research question positively: \OurMethod{} consistently converges faster than \IPSSGD{} -- across optimizers, batch sizes and different levels of position bias ($\gamma$).
These findings support the theoretical results obtained in Sections~\ref{sec:analysis} and~\ref{sec:method}.
In some cases, for example when using the \textsc{Adam} optimizer, \OurMethod{} is not only able to converge faster but able to converge to a higher level of performance than \IPSSGD{}.


\section{Conclusion}
\label{sec:conclusion}
We have studied the convergence rate for \acf{SGD} approaches in Counterfactual \acf{LTR}.
A common approach to Counterfactual \ac{LTR} is \ac{IPS}-weighted \ac{SGD}, where the loss or gradients are scaled by \ac{IPS} weights.
We prove that, for \ac{IPS}-weighted \ac{SGD}, the \ac{IPS} weights play an important role in the convergence rate: the time to converge is slowed by a factor $\BigO{M^2}$ where $M$ is the maximum \ac{IPS} weight in the dataset.

To overcome the slow convergence of \ac{IPS}-weighted \ac{SGD} we propose a sample-based Counterfactual \ac{LTR} learning algorithm called \OurMethod{}.
We prove that \OurMethod{} reduce the convergence rate slowdown from $\BigO{M^2}$ to $\BigO{\bar{M}^2}$ where $\bar{M}$ is the average \ac{IPS} weight in the dataset.
When $M \gg \bar{M}$, this improvement leads to significantly faster convergence of the learning algorithm.

We support our findings with extensive experimentation across a number of biased LTR scenarios, comparing \OurMethod{} to \ac{SGD} with and without \ac{IPS} weighting.
In all cases \OurMethod{} is able to converge faster than standard \ac{IPS}-weighted \ac{SGD}.
In some scenarios \OurMethod{} is even able to converge to a better level of performance than \ac{IPS}-weighted \ac{SGD}.

There are several directions for future work:
First, the convexity assumptions made in the analysis may not hold in practice, particularly when implementing deep neural networks.
Showing the convergence rate of \ac{IPS}-weighted \ac{SGD} for non-convex problems remains an open problem.
Second, optimizing an \ac{IPS}-weighted objective is arguably the simplest approach to Counterfactual \ac{LTR} and in future work we would like to consider more sophisticated objectives such as self-normalized \ac{IPS}~\cite{swaminathan2015self} and variance regularization~\cite{swaminathan2015counterfactual}.
Third, our experiments are conducted on click simulations, giving us experimental control to test our hypotheses.
We leave applying \OurMethod{} to large-scale industrial click logs as future work.
Finally, our work assumes that propensity scores are known a priori which is not always realistic.
Robustness against misspecified propensity scores remains an open problem.

\section*{Code and data}
To facilitate reproducibility of our work, we are sharing all resources used in this paper at \href{https://github.com/rjagerman/sigir2020}{http://github.com/rjagerman/sigir2020}.

\begin{acks}
We thank Chang Li, Harrie Oosterhuis and Ilya Markov for helpful discussions and feedback.
We thank the anonymous reviewers for their feedback.
This research was partially supported by the Netherlands Organisation for Scientific Research (NWO) under project nr 612.001.551 and the Innovation Center for AI (ICAI).
All content represents the opinion of the authors, which is not necessarily shared or endorsed by their respective employers and/or sponsors.
\end{acks}

\bibliographystyle{ACM-Reference-Format.bst}
\bibliography{references}

\end{document}